\documentclass{article}

\usepackage{amssymb,amsmath,bm,dsfont,amsthm}
\usepackage{textcomp}
\usepackage{pdfpages}
\usepackage{color,soul}
\usepackage[ruled,vlined]{algorithm2e}
\newtheorem{theorem}{Theorem}
\newtheorem{definition}{Definition}
\newtheorem{lemma}{Lemma}
\newtheorem{corollary}{Corollary}
\usepackage{multirow}
\usepackage{hyperref,cleveref}
\usepackage{lineno}

\makeatletter
\newcommand{\algorithmfootnote}[2][\footnotesize]{%
	\let\old@algocf@finish\@algocf@finish
	\def\@algocf@finish{\old@algocf@finish
		\leavevmode\rlap{\begin{minipage}{\linewidth}
				#1#2
		\end{minipage}}%
	}%
}
\makeatother

\usepackage{arxiv}

\usepackage[utf8]{inputenc} 
\usepackage[T1]{fontenc}    
\usepackage{hyperref}       
\usepackage{url}            
\usepackage{booktabs}       
\usepackage{amsfonts}       
\usepackage{nicefrac}       
\usepackage{microtype}      
\usepackage{lipsum}		
\usepackage{graphicx}
\usepackage{natbib}
\usepackage{doi}

\title{IGANI: Iterative Generative Adversarial Networks for Imputation with Application to Traffic Data}

\author{ Amir Kazemi\\
	Department of Civil and Environmental Engineering\\ University of Illinois at Urbana-Champaign\\
	Urbana, IL 61801, USA.\\
	\texttt{kazemi2@illinois.edu} \\
	\And
	Hadi Meidani\thanks{Corresponding author}\\
	Department of Civil and Environmental Engineering\\ University of Illinois at Urbana-Champaign\\
	Urbana, IL 61801, USA.\\
	\texttt{meidani@illinois.edu} \\
}

\renewcommand{\shorttitle}{IGANI: Iterative Generative Adversarial Networks for Imputation}

\begin{document}
\maketitle

\begin{abstract}
	Increasing use of sensor data in intelligent transportation systems calls for accurate imputation algorithms that can enable reliable traffic management in the occasional absence of data. As one of the effective imputation approaches, generative adversarial networks (GANs) are implicit generative models that can be used for data imputation, which is formulated as an unsupervised learning problem. This work introduces a novel iterative GAN architecture, called Iterative Generative Adversarial Networks for Imputation (IGANI), for data imputation. IGANI imputes data in two steps and maintains the invertibility of the generative imputer, which will be shown to be a sufficient condition for the convergence of the proposed GAN-based imputation. The performance of our proposed method is evaluated on (1) the imputation of traffic speed data collected in the city of Guangzhou in China, and the training of  short-term traffic prediction models using imputed data, and (2) the imputation of multi-variable traffic data of highways in Portland-Vancouver metropolitan region which includes volume, occupancy, and speed with different missing rates for each of them. It is shown that our proposed algorithm mostly produces more accurate results compared to those of previous GAN-based imputation architectures.
\end{abstract}


\section{Introduction}
\label{sec:introduction}
Traffic data is generated faster than before as intelligent transportation systems (ITS) develop.
Various and numerous operating sensors from road-side cameras to vehicles equipped with GPS tracking devices underscore the big-data features of traffic records.
Traffic management and control highly depend on the collected data which is usually incomplete because of sensor or transmission failures\citep{li2014missing,wu2019imputation}.
Imputation of traffic data is therefore crucial for an informative presentation of them as well as for training prediction models based on them.
The choice of method for incomplete data analysis is of critical importance as it directly affects the conclusion validity.
The methods for handling missing data ranges from naive deletion of incomplete samples to modern machine learning techniques for imputation.
The propriety of a method for missing data analysis depends on the missing rate, missingness mechanism, data type and size, and the acceptable accuracy of imputation or of downstream classification and regression tasks based on the imputed data.
\citep{little2019statistical}.

Myriads of techniques have been proposed for the imputation of missing data which are categorized into statistical and machine-learning based techniques as well as those which are tailored for traffic data to address their spatial-temporal correlations.
In the past decade, statistical techniques including mean/mode, linear regression, least square, and expectation maximization have been outperformed by machine-learning based techniques with respect to accuracy and applicability to different missing mechanisms.
These statistical methods, such as the mean/mode method,  are now being implemented to rather serve as baseline for comparative studies \citep{lin2020missing,baraldi2010introduction,murray2018multiple}.
Major machine learning techniques for data imputation include, but are not limited to, clustering-based methods (e.g. KNN and K-means) \citep{zhang2012nearest,beretta2016nearest}, Support Vector Machine (SVM), Denoising Autoencoders (DAEs) \citep{costa2018missing}, and Generative Adversarial Networks (GANs) \citep{pmlrv80yoon18a,li2018learning,kachuee2019generative,shang2017vigan,cai2018deep,luo2018multivariate}.
Specialized imputation methods for traffic data are majorly matrix/tensor repair methods which use statistical techniques such as matrix/tensor decomposition
\citep{tan2013tensor,asif2016matrix,chen2018spatial,chen2019bayesian}.
Tensor-based methods have been proposed to overcome the shortcoming of matrix based ones in the face of high missing rates as tensors are capable of preserving multi-way features of traffic data
\citep{wu2019imputation}.
Despite the low-rank property of traffic data sets, which is the foundation for tensor decomposition methods, the choice of rank and decomposition method is subjective and challenging especially for incomplete data \citep{chen2019bayesian}.

Further development of imputation techniques using GAN is highly motivated by ITS.
First, the accuracy of GAN-based imputation architectures have been significantly competitive compared with both statistical and other machine-learning techniques \citep{pmlrv80yoon18a,li2018learning,kachuee2019generative,shang2017vigan,cai2018deep,luo2018multivariate}, so they are appealing to ITS where poor traffic management based on erroneous data may cost millions of dollars \citep{zambrano2020intelligent}.
Second, traffic data are usually high dimensional which can be represented by implicit models like GAN avoiding the bias of paramateric distributions\citep{goodfellow2016nips}.
Third, a generative imputer once trained can be applied to a single data, unlike clustering-based methods which must search a sample of data for finding neighbors of the incomplete one.
Finally, GAN models like every DNN can be fine-tuned efficiently as ITS data are generated in large volumes and velocity.
Despite all the advantages of the GAN-based imputation methods, they suffer from inaccuracy for high missing rates, and/or complicated architecture. 
Therefore, more exploration is needed for traffic data imputation using GAN in ITS.

The contribution of this work is as follows:
A novel GAN-based architecture for data imputation is proposed which outperforms previous GAN-based architectures with respect to accuracy and simplicity of the architecture.
The architecture, called Iterative Generative Adversarial Networks for Imputation (IGANI), is simple as it enjoys one generator and discriminator.
The generator is iterated over imputed data and the intuition behind such iteration is to train a robust discriminator which can identify the first-hand imputed data as real and the second-hand one as fake.
We will demonstrate how our proposed method compare with competing imputation methods in accurately imputing missing traffic (speed) data.
As a comparison basis, we also demonstrate that the imputed traffic data from our method can better train a short-term traffic prediction model, compared with other GAN-based imputation methods, and that this superiority holds for different missing rates.
Also, the imputation is performed and evaluated for multi-variable traffic data consisting of volume, occupancy, and speed with different missing rates for each of them.

The paper is organized as follows. First, a brief background for generative models, GANs, and GAN-based imputation architectures is provided in Section~\ref{sec.bh}. Then in Section~\ref{sec.method} we propose the architecture of IGANI. Section~\ref{sec.results} includes the application of the proposed imputation method on traffic data and the discussion of the results.

\section{Background}\label{sec.bh}

\subsection{Generative Models}
Generative models are models that generate data, either explicitly or implicitly.
Explicit generative models considers the data to follow a density function $p_{\bm{\theta}}(\bm{x})$ whose parameters $\bm{\theta}$ are  estimated through maximum likelihood method using the log-likelihood function $\log p_{\bm{\theta}}(\bm{x})$.
Though a computationally intractable log-likelihood may be substituted by other objectives like Jensen-Shannon divergence (JSD) \citep{theis2015note}.
Explicit generative models are also known as prescribed probabilistic models, because the parametric distribution is dictated in contrast to implicit models which merely generate data \citep{mohamed2016learning}.

Selecting computationally tractable densities is an important step in explicit generative models, while the choice of a density function that is capable of capturing data complexity is not straightforward.
Tractable densities are modeled either by fully visible belief networks (FVBNs) \citep{frey1998graphical} or non-linear independent component estimation (NICE) \citep{dinh2014nice}.
FVBNs are based on the chain rule of probability as $p(\bm{x})= p(x_1)p(x_2|x_1)\cdots p(x_n|x_1,x_2,\dots,x_{n-1})$ which gives samples entry-by-entry at a cost of $\mathcal{O}(n)$ for each sample.
NICE considers $p(\bm{x})$ to be a continuous, invertible, nonlinear transformation which maps a latent variable $\bm{z}$ to $\bm{x}$, i.e. $\bm{x}=g(\bm{z})$.
Then $p_x(\bm{x}) = p_z(g^{-1}(\bm{x}))|\det (\partial g^{-1}(\bm{x})/ \partial \bm{x})|$ becomes a tractable density function, if $p_z$ and the determinant is tractable.
Data generation by FVBNs is time consuming and cannot be parallelized and NICE requires $g$ to be invertible with $\bm{x}$ and $\bm{z}$ having the same dimension \citep{goodfellow2016nips}.
If the explicit density is not tractable, variational methods can be employed. A typical example is variational autoencoders (VAE) which uses a tractable lower bound  for an intractable log-likelihood \citep{rezende2015variational}. Another option is Boltzman machines which is based on Markov Chain Monte Carlo \citep{hinton1984boltzmann}.
Boltzman machines simulate a sequence of samples  $\bm{x}'\sim q(\bm{x}'|\bm{x})$ where $q$ is a transition probability density designed in such a way that the distribution of the samples will converge to  $p(\bm{x})$.
While variational methods like VAEs are affected by the accuracy of the posterior or prior distributions, MCMC methods, such as Boltzman machines, suffer from slow convergence \citep{goodfellow2016nips}.

As opposed to explicit methods, implicit generative models do not parameterize the density of data, but are still appealing for high-dimensional data representation, model-based reinforcement learning, data imputation, models with multi-modal outputs, and realistic data generation \citep{goodfellow2016nips}.
Implicit methods train functions which map from a latent variable to the data space; though the mapping is deterministic, the latent variable is the external stochastic source \citep{mohamed2016learning}.
Implicit models ranges from basic non-uniform random variate generator \citep{devroye1996random} to Generative Stochastic Networks (GSNs) \citep{bengio2014deep} and Generative Adversarial Networks (GANs) \citep{goodfellow2014generative}.
GSNs using MCMC are time-consuming especially for high-dimensional data as they use Markov chains, 
while GANs do not have most of the aforementioned shortcomings. First, data generation in GANs is performed in parallel, independently of the dimension of $\bm{x}$.
Second, Markov chains, tractable densities, invertible mapping from latent variable $\bm{z}$, and variational bounds are not required by GANs \citep{goodfellow2016nips}. In the next section we present a technical background on GANs.

\subsection{Generative Adversarial Network (GAN)}

Generative Adversarial Networks (GANs) provide a game-theoretic framework for obtaining the implicit distribution of data.
This is performed by adversarial networks where a discriminator network estimates the probability of a data instance being real or fake.
The data instances are mapped from a latent variable $\bm{z}$ by a generative model called generator which is trained to deceive the discriminator.
The discriminator is therefore trained in a supervised way, where real and fake data are labeled as one and zero, respectively \citep{goodfellow2014generative}.

Let $G$ and $D$ be the generator and discriminator, each a neural network that is differentiable with respect to its respective input, $\bm{z}$ and $\bm{x}$, and also with respect to their respective neural parameters,  $\bm{\theta}^{(G)}$ and $\bm{\theta}^{(D)}$. The generator $G$ for each $d$-dimensional (noise) input $\bm{z}$ returns a $d$-dimensional output $\hat{\bm{x}}$, and the discriminator $D$ for each input data (either $\bm{x}$ or $\hat{\bm{x}}$) returns a scalar indicating (deterministically or probabilistically) whether the data is real or fake.
The training involves minimizing  two cost functions associated with the generator and discriminator, namely $J^{(G)}(\bm{\theta}^{(G)},\bm{\theta}^{(D)})$, and $J^{(D)}(\bm{\theta}^{(G)},\bm{\theta}^{(D)})$. Note that both functions depends on both $\bm{\theta}^{(G)}$ and $\bm{\theta}^{(D)}$. In each iteration of the training, a minibatch of inputs, i.e. $\bm{z}$ and $\bm{x}$ are selected for $D$ and $G$ and gradient-based optimization minimizes the cost functions of the generator and discriminator subsequently. The architecture of GAN in a training iteration is illustrated in Fig. (\ref{fig.gan}) for the mini-batch size of 3 and a 5-dimensional data.

\begin{figure}[h]
	\centering
	\includegraphics[width=3.5in]{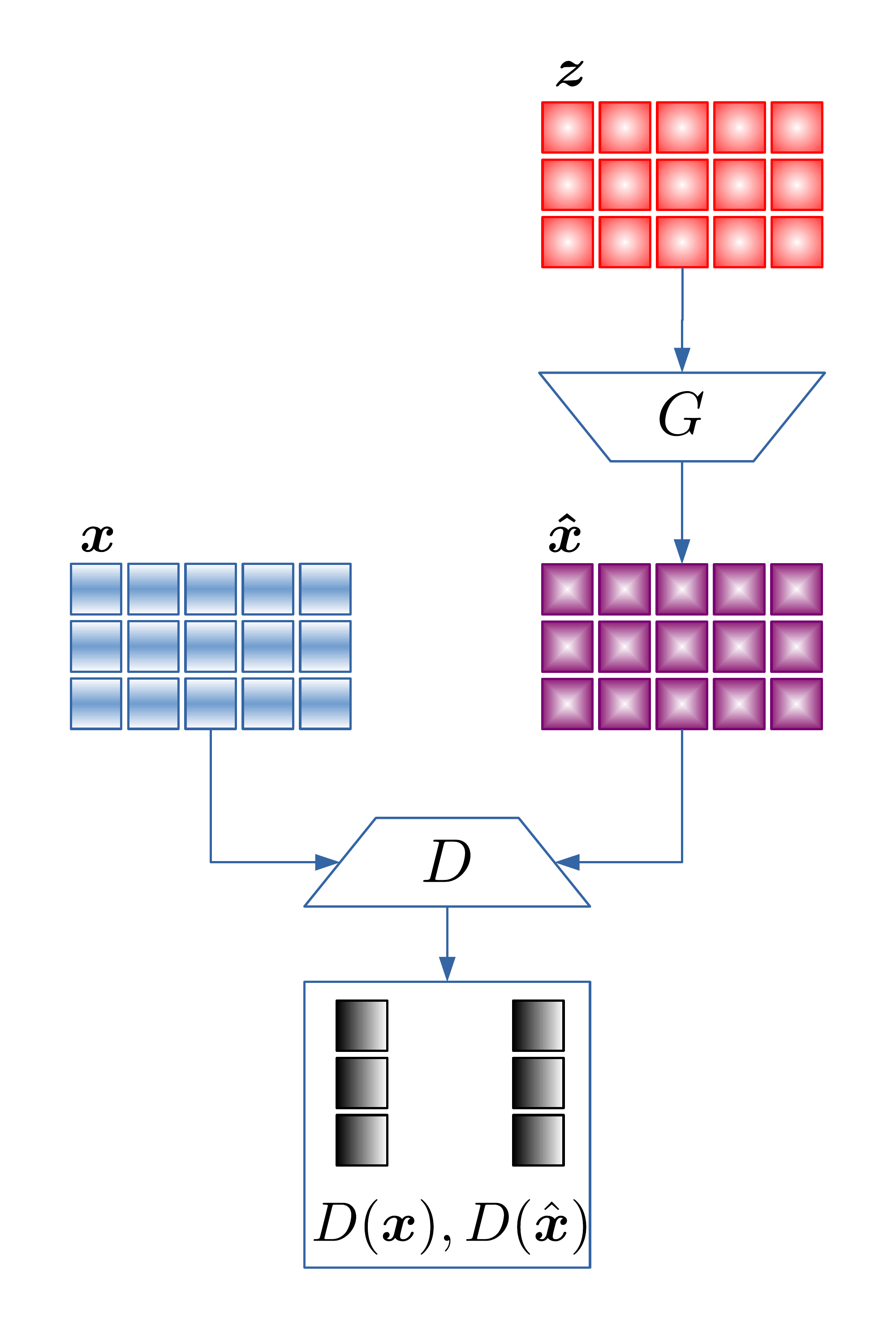}
	\caption{The GAN architecture, where the generator $G$ generates data $\hat{\bm x} = G(\bm{z})$  from the latent variable $\bm{z}$  and the discriminator $D$ determines whether  $\hat{\bm x}$ or  $\bm{x}$ is fake or not. $G$ and $D$ are adversarial networks, i.e. $G$ is trained toward deceiving $D$ and  as a result the distribution of the generated data $\hat{\bm x}$ tends to that of $\bm x$. The number of columns in $\bm{x}$, $\bm{z}$, and $\hat{\bm{x}}$ equals the dimension of data and their number of rows is the number of samples (or mini-batches). }  \label{fig.gan}
\end{figure}

The discriminator parameters $\bm{\theta}^{(D)}$ are estimated  by minimizing $J^{(D)}(\bm{\theta}^{(G)},\bm{\theta}^{(D)})$, while  $\bm{\theta}^{(G)}$ are estimated by minimizing $J^{(G)}(\bm{\theta}^{(G)},\bm{\theta}^{(D)})$. The whole minimization is considered as a game, because each network can control only its own parameters \citep{goodfellow2016nips}. The solution of such a game is that of Nash equilibrium \citep{ratliff2013characterization}, i.e. the pair $(\bm{\theta}^{(G)},\bm{\theta}^{(D)})$ which locally minimizes $J^{(D)}(\bm{\theta}^{(G)},\bm{\theta}^{(D)})$ with respect to $\bm{\theta}^{(D)}$, and minimizes $J^{(G)}(\bm{\theta}^{(G)},\bm{\theta}^{(D)})$ with respect to $\bm{\theta}^{(G)}$.

Different schemes for cost functions can be used in GANs. These schemes usually differ in $J^{(G)}$, rather than in  $J^{(D)}$ \citep{goodfellow2016nips}.
One simple choice for the discriminator's cost function is given by
\begin{equation}\label{lossD}
		J^{(D)}(\bm{\theta}^{(G)},\bm{\theta}^{(D)})  = 
		 -\frac{1}{2} \mathbb{E}_{\bm x} \log(D(\bm{x}))
		 - \frac{1}{2} \mathbb{E}_{\bm z} (1-\log(D(G(\bm{z}))).
\end{equation}
where $\mathbb{E}_x$ denote the expected (or mean) value calculated over the distribution of $x$. In the zero-sum game, the cost of generator is selected such that it neutralizes the cost of discriminator, i.e.,
\begin{equation}
	J^{(G)} = -J^{(D)}.
\end{equation}
This is equivalent to a minimax optimization as follows
\begin{equation}
	\bm{\theta}^{(G)*} =
	\arg \min_{\bm{\theta}^{(G)}} 
	\max_{\bm{\theta}^{(D)}}
	- J^{(D)}(\bm{\theta}^{(G)},\bm{\theta}^{(D)}),
\end{equation}
which does not perform well in practice because both $D$ and $G$ are minimizing and maximizing the same cross-entropy, respectively.
Instead of minimizing the probability of a correct discrimination, to avoid vanishing gradient issue in the optimization, $G$ can be trained by maximizing the probability of a wrong discrimination, or effectively by minimizing the following cost function
\begin{equation}\label{lossG}
	J^{(G)}(\bm{\theta}^{(G)},\bm{\theta}^{(D)}) = -\frac{1}{2} \mathbb{E}_z \log(D(\bm{z})).
\end{equation}

The role of discriminator in predicting the probability of being real or fake data can be relaxed using Wasserstein distance in the cost functions.
Wasserstein or Earth-Mover distance facilitate the convergence of a probability distribution sequence for which popular measures for distributional distance like Jensen-Shannon divergence, Kullback-Leibler divergence do not converge.
Therefore, training GANs using Wasserstein distance (WGANs) is more stable and the results are less sensitive to hyper-parameters and architectures \citep{arjovsky2017wasserstein}.
WGANs uses a linear activation output layer instead of sigmoid function for the discriminator (critic), employing the Wasserstein distance in the cost functions, and constraining the parameters (namely weights) of discriminator to fall into a compact support by clipping. 
Clipping underuses the capacity of discriminator and, if not tuned well, may lead to exploding or vanishing gradient.
This can be avoided by incorporating a gradient penalty (GP) in WGAN Lipschitz constraint on the discriminator loss \citep{gulrajani2017improved} as follows
\begin{equation}\label{wgan.Dloss}
		J^{(D)}_W(\bm{\theta}^{(G)},\bm{\theta}^{(D)})  =  
		\mathbb{E}_x D(\bm{x}) - \mathbb{E}_z D(G(\bm{z})) \\
		+ \lambda \mathbb{E}_y ( \Vert \nabla_{\bm{y}} D(\bm{y}) \Vert - 1)^2
\end{equation}
where $\bm{y}=tG(\bm{z})+(1-t)\bm{x}$ with $0\leq t \leq 1$ and $\lambda$ is a coefficient, usually set as $\lambda=10$.
The cost function of generator like Eq. (4) can be written as
\begin{equation}\label{wgan.Gloss}
	J^{(G)}_W(\bm{\theta}^{(G)},\bm{\theta}^{(D)})=
	-\mathbb{E}_z D(\bm{z}).
\end{equation}
The proposed algorithm in this work uses WGAN-GP with the cost functions defined in  Eqs. (5-6).

\subsection{GAN-based imputation methods}

As mentioned above, GAN is an implicit generative model which can be applied to data imputation.
A pioneering method is Generative Adversarial Imputation Network (GAIN) \citep{pmlrv80yoon18a} which has a sweet epoch in the training phase urging early-stoping. GAIN has not been successful on block missing patterns. For these patterns, another architecture, called MisGAN, has been shown to be more effective \citep{li2018learning}.
In MisGAN, three generators for mask, complete data, and imputed data are simultaneously trained together with their corresponding discriminators.
In~\citep{kachuee2019generative}, the performance of MisGAN was questioned in uniform missing pattern and a Generative Imputation (GI) was introduced as an accurate method for both block and uniform missing patterns.
The problem of imputing multi-view and multi-modal data which are observed from heterogeneous sources is addressed in \citep{shang2017vigan,cai2018deep}. 
GAN is also leveraged by Recurrent Neural Network (RNN) to impute incomplete multivariate time series in \citep{luo2018multivariate}, and demonstrated accuracy improvement over GAIN and MisGAN.
In this work, we used GAIN and MisGAN as the prominent competing architectures against which our proposed IGANI method is compared. 
Before that, we define what a generative imputer is:

\begin{definition}\label{defimputer}
	Let $\bm x\in \mathbb{R}^d$ denote a random vector and $\bm m\in \lbrace 0,1 \rbrace^d$ be its random mask where $m_j=0$ or $m_j=1$ means that $x_j$ is observed or missing, respectively.
	A generative imputer is defined as $(\bm u,\bm v)=G(\bm x,\bm m,\bm z)$ where
	\begin{equation}\label{imputer}
		\begin{split}
			\bm u & = \bm x\odot \bm m + \bm z\odot (1-\bm m) \\
			\bm v & = \bm x\odot \bm m + g(\bm u)\odot (1-\bm m) 
		\end{split}
	\end{equation}
	where $\bm z\in \mathbb{R}^d$ is noise and $g(\cdot)$ is a function to be learned, and $\bm v$ is the imputed vector where $v_j = x_j$ if $x_j$ is observed.
	Alternatively, one may write $(\bm u,\bm v)=G(\bm x^{(m)},\bm m,\bm z^{(1-m)})$ where $\bm x^{(m)} = \bm x\odot \bm m$ and $\bm z^{(1-m)} = \bm z\odot (1-\bm m)$.
	
\end{definition}

\begin{figure}[]
	\centering
	\includegraphics[scale=0.35]{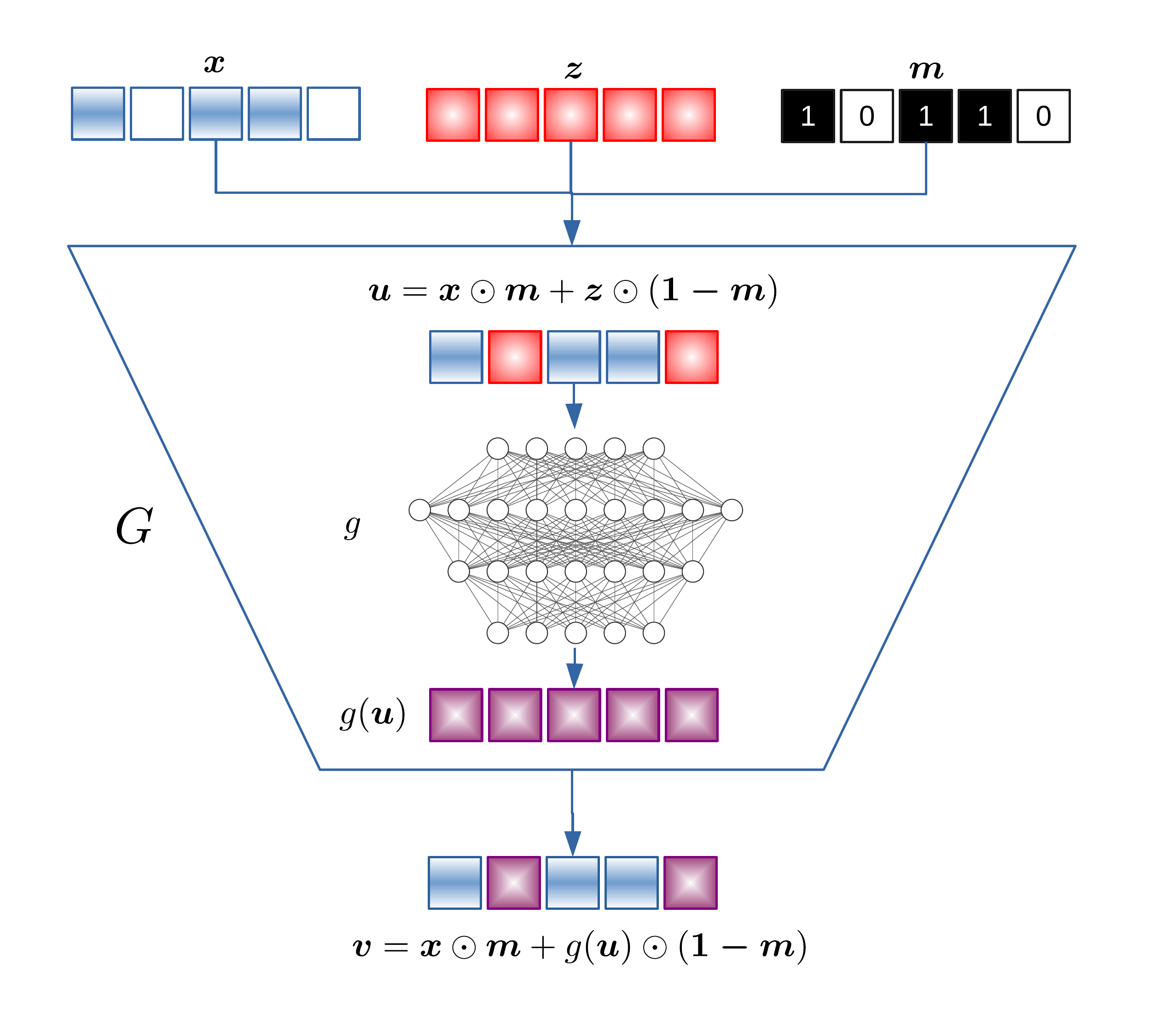}
	\caption{Generative imputer: blue and white squares denote observed and missing values of incomplete data $\bm x$. Noise is represented by red squares as shown for the noise vector $\bm z$. The generated data $g(\bm u)$ is shown in purple. Note that the imputed vector $\bm v$ shares the observed values with $\bm x$.} \label{fig.genimp}
\end{figure}

In GAIN, shown in Fig.~\ref{fig.gain}, a discriminator $D$ maximizes the probability of predicting mask $\bm m$ based on the imputed data $\bm v$.
Prediction of the mask can also be interpreted as how real/fake every element of the imputed vector is.
This generalizes the role of conventional discriminators in GANs which give a scalar score for the whole generated vector.
Mask prediction is realized by using a partial information about the original mask matrix which is referred to as the hint $\bm h$.
Specifically in GAIN, the hint mechanism chooses one element of each row of the mask matrix randomly and set it to be 0.5.
Using the imputed data $\bm v$ and the hint $\bm h$, the discriminator $D$ predicts the mask as $\hat{\bm m} = D(\bm h, \hat{\bm x})$ whose distance from $\bm m$ is minimized. Using the Sigmoid function as the last layer, we force the entries of $\hat{\bm m}$ to be in $[0,1]$.
The cost functions for the discriminator and generator in Eqs. (\ref{lossD}, \ref{lossG}) are therefore modified as:
\begin{equation}
		J^{(D)}(\bm{\theta}^{(G)},\bm{\theta}^{(D)})  = - \mathbb{E}_{\bm  m,\hat{\bm m}}  \big[ \sum_{i=1}^d [ \bm  m_i\log(\hat{\bm m}_i)+
		(1-\bm m_i)\log(1-\hat{\bm m_i}) ] \big] 
\end{equation}
and
\begin{equation}
	\begin{split}
		J^{(G)}(\bm{\theta}^{(G)},\bm{\theta}^{(D)}) & =
		- \mathbb{E}_{\bm  m,\hat{\bm m}} \left[ \sum_{i=1}^d (1-\bm m_i)\log(\hat{\bm m_i}) \right]\\
	\end{split}
\end{equation}

\begin{figure}[]
	\centering
	\includegraphics[width=3in]{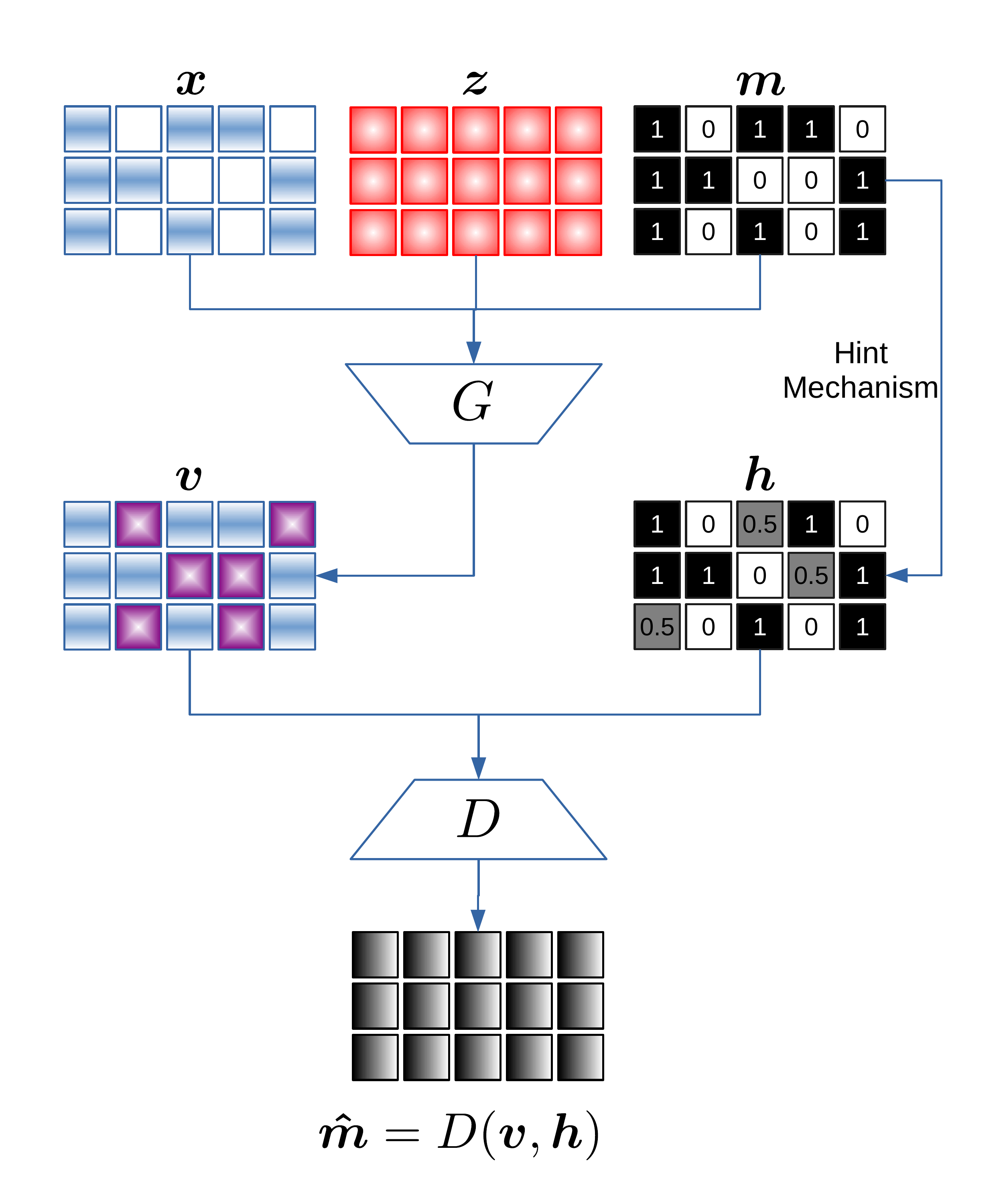}
	\caption{GAIN architecture: the generative imputer $G$ is trained along with the discriminator $D$ which predicts the mask ($\bm m$) using the imputed data ($\bm v$) and the hint ($\bm h$).} \label{fig.gain}
\end{figure}

MisGAN, compared to GAIN, has a more complicated architecture and includes three pairs of generators and discriminators trained simultaneously by the WGAN loss functions in Eqs (\ref{wgan.Dloss}) and (\ref{wgan.Gloss}).
Generators are $G_m$, $G_x$, and $G_i$ which generate mask, data, and imputed data, respectively.
Mask generation is the first step, because we access to the real mask matrix whose distribution can be implicitly recovered by training $G_m$ within a GAN scheme.
The same procedure can be applied to train $G_x$, i.e. the generator of observed data represented by the mask function $f(\bm x,\bm m)=\bm x\odot \bm m + \tau (1-\bm m)$.
Note that the arbitrary constant $\tau$ is a proxy for the unobserved data and the choice of its value does not affect the convergence.
Given the distribution of mask and data as described, one may train the imputer generator ($G_i$) whose output ($v$) is discriminated from the data generated by $G_x$, i.e. ($\hat{\bm x}$).

\begin{figure}[]
	\centering
	\includegraphics[width=3.5in]{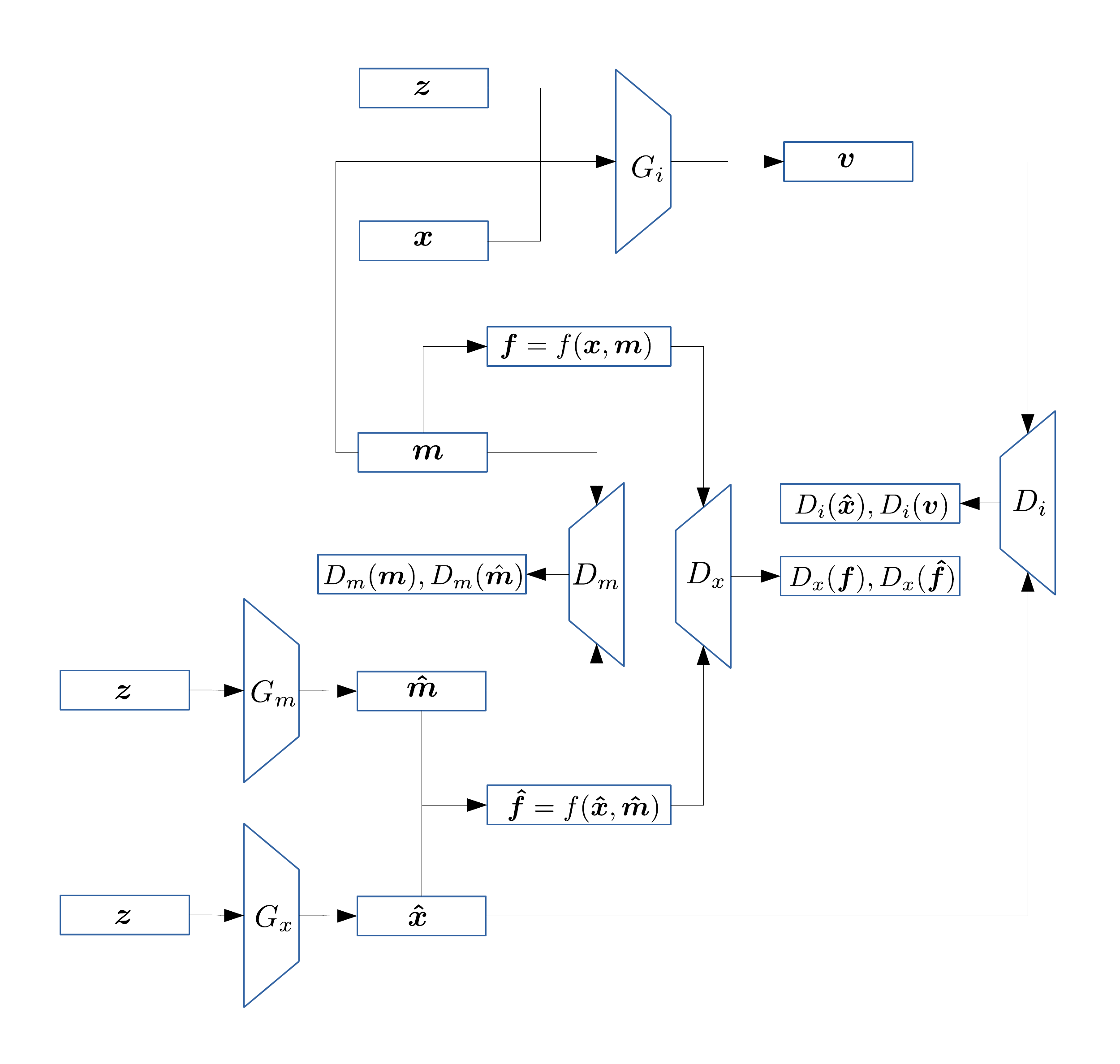}
	\caption{MisGAN architecture: three pairs of generators and discriminators are trained simultaneously. $G_m$, $G_x$, and $G_i$ generate mask, data, and imputed data, respectively. A mask function $f(\bm x,\bm m)=\bm x\odot \bm m + \tau (1-\bm m)$ is used to train the data generator $G_x$ and its corresponding discriminator ($D_x$). Note that $G_i$ denotes the generative imputer and $\bm v$ is the imputed data.} \label{fig.MisGAN}
\end{figure}

\section{Method}\label{sec.method}

In this section, we propose the architecture of IGANI as a GAN-based data imputation method.
As can be seen in Fig. \ref{fig.IGANI},  this approach  first applies the generative imputer $G$  to produce the `imputed' data $\bm v$. Then  in an iteration,  it again applies  $G$, this time on $\bm v$ and a reshuffled mask matrix, to produce a `re-imputed' data  $\hat{\bm v}$.
Note that ${\bm v}$ has more observed values compared to $\hat{\bm v}$.
Then, the discriminator $D$ seeks to distinguish between `imputed' and `re-imputed' data, by giving scalar scores to  $\bm v$ and  $\hat{\bm v}$ vectors separately, as done in conventional GANs.
We train IGANI, or the two neural networks $G$ and $D$ specifically, by minimizing  the following  WGAN loss functions
\begin{equation}\label{IGANIlosses}
	\begin{split}
		J^{(D)}_W &= \mathbb{E}_v D(\bm{\bm{v}}) - \mathbb{E}_z D(\bm{\hat{v}}) + \lambda \mathbb{E}_y ( \Vert  \nabla_{\bm{y}} D(\bm{y}) \Vert - 1)^2,\\
		J^{(G)}_W &= - \mathbb{E}_z D(\bm{\hat{v}}),
	\end{split}
\end{equation}
where $\bm{y}=tG(\bm{z})+(1-t)\bm{v}$ with $0\leq t \leq 1$ and $\lambda = 10$.

According to the architecture in Fig. \ref{fig.IGANI}, the distribution of $\hat{\bm v}$ tends to that of ${\bm v}$, if the GAN converges.
However, the ultimate goal is to show that the distribution of the imputed data with the new mask ($\bm v^{(n)} = \bm v \odot \bm n$) tends to that of the observed data ($\bm x^{(n)} = \bm x \odot \bm m$).
Briefly speaking, we want to show that $p_{\hat{\bm v}} \rightarrow p_{\bm v}$ gives $p_{\bm v^{(n)}} \rightarrow p_{\bm x^{(m)}}$.
This is guaranteed by the invertiblity of the generative imputer $G$ in Def. (\ref{defimputer}) which is preceded by showing that $\bm m$ is recovered almost everywhere from $(\bm u, \bm v)$.
This is rather intuitive; as Fig.~\ref{fig.genimp} shows, $\bm v$ differs from $\bm u$ only for missing values:

\begin{figure}[!h]
	\centering
	\includegraphics[width = 3in]{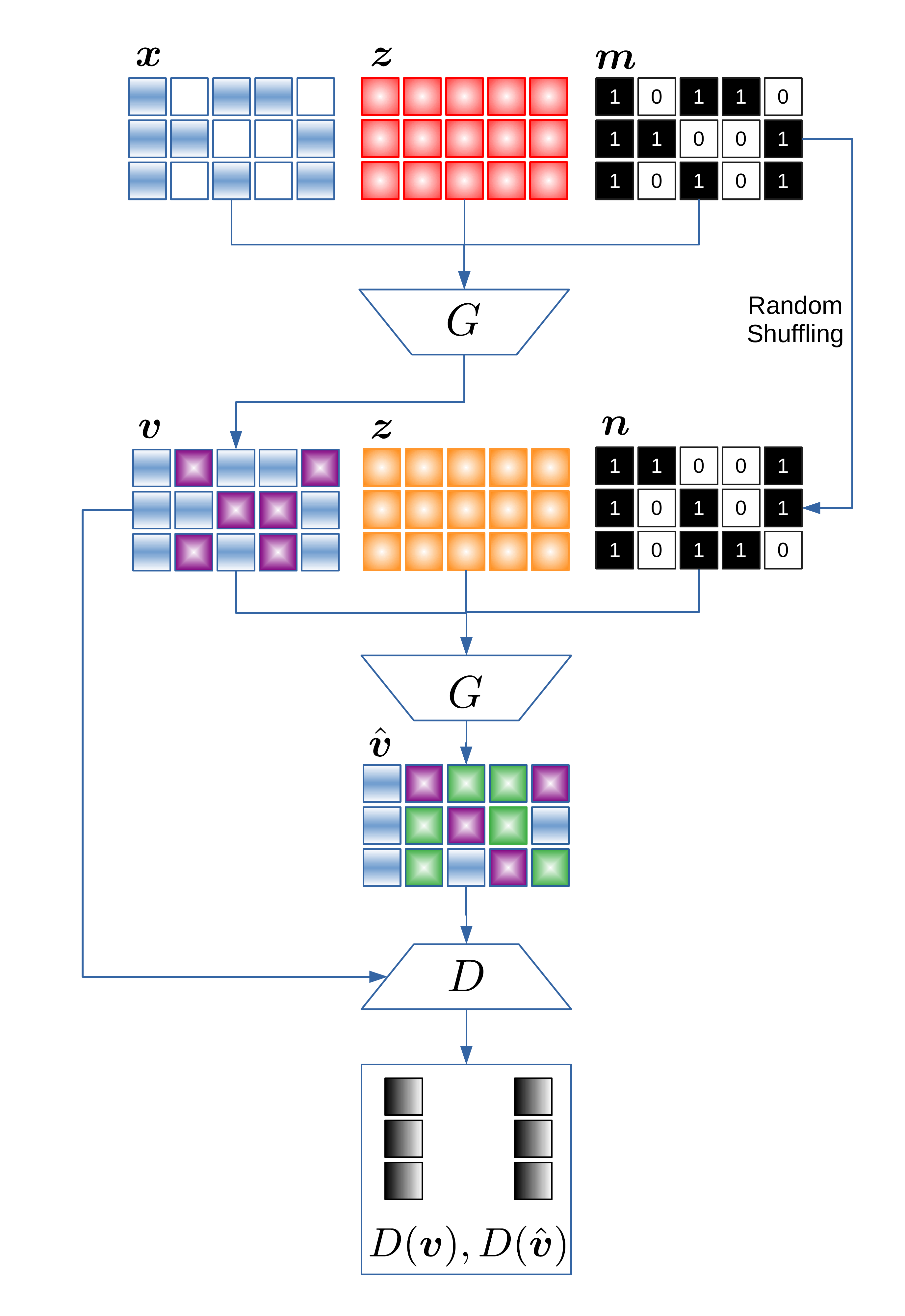}
	\caption{IGANI architecture: iteration of the generative imputer $G$ gives a second-hand imputed data $\hat{\bm v}$ which has less observed values compared with the first-hand imputed data $\bm v$; compare the blue squares between $\bm v$ and $\hat{\bm v}$. So $G$ and $D$ are trained to compensate the decline of observed values after the secondary imputation, by means of imputing more accurately.} \label{fig.IGANI}
\end{figure}

\begin{lemma}\label{gnonlinear}
	For a nonlinear function $g$ in Eq. (\ref{imputer}), $\bm m=\mathds{1}_{(\bm u=\bm v)}$ holds almost everywhere.
\end{lemma}
\begin{proof}
	If $g$ is a nonlinear function, $g(\bm{u})$ cannot be a linear combination of $\bm{z}$.
	Therefore, the measure of $\lbrace \bm{z} \in \mathbb{R}^d| \exists i \leq d: \bm{z}_i = g(\bm{u})_i\rbrace$ is zero, which means that $\bm{u}$ and $\bm{v}$ are not equal for unobserved indices almost everywhere and $\bm m=\mathds{1}_{(\bm u=\bm v)}$.
	Note that the function sequence $\lbrace g_i\rbrace_{i=1}^{\infty}$ estimating $g$ is trained by deep neural networks which are nonlinear because of their nonlinear activation functions.
\end{proof}

\begin{lemma}\label{Ginvertible}
	The generative imputer $(\bm u,\bm v)=G(\bm x^{(m)},\bm m,\bm z^{(1-m)})$ in Def. 1 is invertible for $\bm m=\mathds{1}_{(\bm u=\bm v)}$.
\end{lemma}
\begin{proof}
	Eq. (\ref{imputer}) is written as
	\begin{equation}\label{imputer2}
		\begin{split}
			\bm u & = \bm x^{(m)} + \bm z^{(1-m)}  \\
			\bm v & = \bm x^{(m)} + g(\bm u) \odot (1-\bm m)
		\end{split}
	\end{equation}
	Then $(\bm x^{(m)},\bm m,\bm z^{(1-m)})$ can be written explicitly in terms of $(\bm u,\bm v)$ as follows
	\begin{equation}
		\begin{split}
			\bm x^{(m)} & = \bm v - g(\bm u)\odot \mathds{1}_{(\bm u\neq \bm v)}  \\
			\bm z^{(1-m)} & = \bm u - \bm v + g(\bm u)\odot \mathds{1}_{(\bm u\neq \bm v)}
		\end{split}
	\end{equation}
	which proves the invertibility of $(\bm u,\bm v)=G(\bm x^{(m)},\bm m,\bm z^{(1-m)})$ for $\bm m=\mathds{1}_{(\bm u=\bm v)}$. 
\end{proof}

Now consider the architecture in Fig. (\ref{fig.IGANI}) which gives $p_{\hat{\bm v}}\rightarrow p_{\bm v}$ provided that GAN converges.
The aim is to show that $p_{\bm v^{(n)}} \rightarrow p_{\bm x^{(m)}}$.
Let $(\bm u,\bm v)=G(\bm x^{(m)},\bm m,\bm z^{(1-m)})$ and $(\hat{\bm u},\hat{\bm v}) = G(\bm v^{(n)},\bm n,\bm z^{(1-n)})$  where $G$ is the generative imputer in Def. 1 and $\bm n\sim p_m$:
\begin{theorem}\label{theo1}
	$p_{\hat{\bm u},\hat{\bm v}}\rightarrow p_{\bm u,\bm v}$ gives $p_{\bm v^{(n)}} \rightarrow p_{\bm x^{(m)}}$.
\end{theorem}
\begin{proof}
	If $p_{\hat{\bm u},\hat{\bm v}}\rightarrow p_{\bm u,\bm v}$, i.e.
	\begin{equation}
		p_{G(\bm v^{(n)},\bm n,\bm z^{(1-n)})} \rightarrow p_{G(\bm x^{(m)},\bm m,\bm z^{(1-m)})}
	\end{equation}
	invertibility of the generative imputer $G$ as shown in Lemma \ref{Ginvertible} gives
	\begin{equation}
		p_{\bm v^{(n)},\bm n,\bm z^{(1-n)}} \rightarrow p_{\bm x^{(m)},\bm m,\bm z^{(1-m)}}.
	\end{equation}
	As the limit holds for joint distributions, it must hold for marginals and $p_{\bm v^{(n)}} \rightarrow p_{\bm x^{(m)}}$.
\end{proof}

\begin{lemma}\label{joint}
	$p_{\hat{\bm v}}\rightarrow p_{\bm v}$ gives $p_{\hat{\bm u},\hat{\bm v}}\rightarrow p_{\bm u,\bm v}$.
\end{lemma}
\begin{proof}
	As shown in Fig. (\ref{fig.genimp}), the generative imputer preserves the observed values of data, i.e. $\bm x^{(m)}=\bm v^{(m)}$ and $\bm v^{(n)}=\hat{\bm v}^{(n)}$. So Eq. (\ref{imputer2}) implies
	\begin{equation}\label{Lem3_1}
		\begin{split}
			p_{\bm u,\bm v} & = p_{\bm v^{(m)}+\bm z^{(1-m)},\bm v } \\
			p_{\hat{\bm u},\hat{\bm v}} & = p_{\hat{\bm v}^{(n)}+\bm z^{(1-n)},\hat{\bm v} }
		\end{split}
	\end{equation}
	Also, let $\gamma(\bm v, \bm m, \bm z) = \bm v - \bm v^{(m)} - g(\bm v^{(m)} + \bm z^{(1-m)}) = \bm 0$ which gives $\gamma(\hat{\bm v}, \bm n, \bm z)= \bm 0$
	because the same generative model $G$ (and $g$) is used in the iterative imputation.
	Therefore, for the same relation $\gamma$, assuming $\bm n\sim p_m$ and $p_{\hat{\bm v}}\rightarrow p_{\bm v}$, it is concluded that $p_{\hat{\bm u},\hat{\bm v}}\rightarrow p_{\bm u,\bm v}$.
\end{proof}

\begin{corollary}\label{corollary}
	$p_{\hat{\bm v}}\rightarrow p_{\bm v}$ gives $p_{\bm v^{(n)}} \rightarrow p_{\bm x^{(m)}}$.
\end{corollary}
\begin{proof}
	Lemma \ref{joint} states $p_{\hat{\bm u},\hat{\bm v}}\rightarrow p_{\bm u,\bm v}$ which according to Theorem \ref{theo1} is enough to have
	$p_{\bm v^{(n)}} \rightarrow p_{\bm x^{(m)}}$.
\end{proof}

The GAN architecture described in Fig.~\ref{fig.IGANI} and implemented in Algorithm (1) imputes the incomplete data $\bm{x}$, provided that GAN gives $p_{\hat{\bm v}} \rightarrow p_{\bm v}$ which according to Corollary \ref{corollary} means $p_{\bm v^{(n)}} \rightarrow p_{\bm x^{(m)}}$. \\

\begin{algorithm}
	Set $N_E$ (number of epochs)\;
	Set $N_{DU}$ (number of discriminator updates)\;
	\While{epoch $\leq N_E$}{
		\For{$(\bm{x},\bm{m}) \in (\bm{X},\bm{M})$}{
			Sample noise $\bm{z}\sim p_z$\;
			\For{$iter \leq N_{DU}$}{
				$(\bm{u},\bm{v}) \gets G(\bm{x},\bm{m},\bm{z})$\;
				$\bm{n} \gets $ a random shuffling of $\bm{m}$\;
				$(\bm{\hat{u}},\bm{\hat{v}}) \gets G(\bm{v},\bm{n},\bm{z})$\;
				Update discriminator $D$ by minimizing the loss $J^{(D)}_W$ in Eqs.~\ref{IGANIlosses}\;
			}
			Update Generator $G$ by minimizing the loss $J^{(G)}_W$ in  Eqs.~\ref{IGANIlosses}.\;
		}
	}
	\caption{Training of IGANI}
	\algorithmfootnote{Training parameters for this paper: learning rate = 1e-4, $N_E = 200$, $N_{DU} = 30 + [\text{epoch}/10]$ }
\end{algorithm}

\section{Experiments}\label{sec.results}
\subsection{Guangzhou (China) Speed Data}
\subsubsection{Data description}
The dataset represents the speed of 214 road segments in Guangzhou (China) collected over two months (from August 1, 2016 to September 30, 2016) at 10-minute time steps. \citep{xinyu_chen_2018_1205229}.
The data is recommended by the providers for data imputation, traffic prediction and pattern discovery.
The spatial and temporal window of the data, as mentioned above, implies a structure of three-dimensional tensor of size $61\times 144 \times 214$, whose dimensions are  day, time, and road segment, respectively.
For learning, the dataset is reshaped as a $8784 \times 214 $ tensor, meaning that the sample size is $8784$.
The original missing rate of data is 1.29\%, which is almost negligible for imputation purposes, and the missing data is calculated using an iterative imputer (that estimates each feature as a function of all the others) before creating a ``complete" reference data.

The described dataset is divided into three portions: (i) 10\% for training the imputers, (ii) 80\% for training the short-term traffic prediction models, and (iii) 10\% for testing the prediction models.
The imputation accuracy is tested on portions (ii) and (iii).

\subsubsection{Architecture}

Network dimensions of the architectures used in the experiments is shown in table \ref{table}.
To have a fair comparison of imputer performance, the same capacity of it is used for them; i.e. $G$ for IGANI and GAIN and $G_i$ for MisGAN have the same layer and dimensions.

\begin{table}
	\caption{Dimension of layers for different architectures}
	\label{table}
	\centering
	\setlength{\tabcolsep}{3pt}
	\begin{tabular}{p{100pt}p{15pt}p{170pt}}
		\hline
		Architecture& 
		DNN& 
		Layer dimensions \\
		\hline
		
		IGANI& 
		$G$ & 
		Dense (214$\times$512), ReLU, Dropout ($p=0.05$), Dense (512$\times$512), ReLU, Dropout ($p=0.05$), and Dense (512$\times$214) \\
		& 
		$D$ & 
		Dense (214$\times$256), ReLU, Dense (256$\times$256), ReLU, and Dense (256$\times$214) \\
		
		GAIN& 
		$G$ & 
		Dense (214$\times$512), ReLU, Dropout ($p=0.05$), Dense (512$\times$512), ReLU, Dropout ($p=0.05$), and Dense (512$\times$214) \\
		& 
		$D$ & 
		Dense ((2$\times$214)$\times$256), ReLU, Dense (256$\times$256), ReLU, and Dense (256$\times$214) \\
		
		MisGAN& 
		$G_m$ & 
		Dense (214$\times$256), ReLU, Dense (256$\times$256), ReLU, and Dense (256$\times$214) \\		
		& 
		$D_m$ & 
		Dense (214$\times$256), ReLU, Dense (256$\times$256), ReLU, and Dense (256$\times$214) \\		
		&
		$G_x$ & 
		Dense (214$\times$256), ReLU, Dense (256$\times$256), ReLU, and Dense (256$\times$214) \\		
		& 
		$D_x$ & 
		Dense (214$\times$256), ReLU, Dense (256$\times$256), ReLU, and Dense (256$\times$214) \\		
		&
		$G_i$ & 
		Dense (214$\times$512), ReLU, Dropout ($p=0.05$), Dense (512$\times$512), ReLU, Dropout ($p=0.05$), and Dense (512$\times$214) \\
		& 
		$D_i$ & 
		Dense (214$\times$256), ReLU, Dense (256$\times$256), ReLU, and Dense (256$\times$214) \\

		Short-term prediction &
		$Net$ &
		Dense (214$\times$424), ReLU, Dropout ($p=0.05$), Dense (424$\times$424), ReLU, Dropout ($p=0.05$), and Dense (424$\times$214) \\
		\hline
		\multicolumn{3}{p{295pt}}{Note: The dimension of inputs for Guangzhou data is always 214 for all the networks. The dimension of Portland-Vancouver data (discussed in Section~\ref{sec.portland}) is  480, which replaces 214 in this table.}\\
		\hline
	\end{tabular}
	\label{tab1}
\end{table}

\subsubsection{Imputation of missing traffic data}

The proposed method in this work, IGANI, is applied to the mentioned data with different missing rates.
The missing mechanism is MCAR which is the common case for incomplete traffic data.
The performance of IGANI with respect to test accuracy is compared with two recent GAN-based imputation methods, i.e. GAIN and MisGAN, and the results are shown in Fig.~\ref{fig.maeimp}. For each of the three approaches, we ran 5 separate trainings, and the mean and standard deviation results in Fig.~\ref{fig.maeimp} are calculated using 5 separate tests on the 5 trained models.

As can be seen, the mean absolute error (MAE) of imputation by IGANI, compared with other methods, is  lower for all the missing rates.
Also, GAIN and MisGAN are less accurate than the baseline (mean imputation) for missing rates higher than 70\%, while IGANI always outperforms the baseline.
Superiority of IGANI is also visible for an individual sample of data; see Fig.~\ref{fig.maeimpinstance}.
The outperformance of IGANI may be justified as follows.
MisGAN architecture, as can be seen from Fig. \ref{fig.MisGAN}, is complicated with three pairs of generator-discriminator trained together at every epoch.
The discriminator of imputer ($D_i$) treats the output of data generator ($G_x$) as the real data ($\hat{\bm x}$), which is far from being real in the early stages of training and mislead the imputer.
Such an adverse early-stage effect of NNs on each-other, does not occur in GAIN and IGANI where both have only one pair of generator and discriminator.
However, the mask hint as an input to the discriminator of GAIN (see Fig. \ref{fig.gain}) is spoiler in the sense that the discriminator has an easier job to distinguish the observed data from the missing one.
The effect of spoiling the discriminator becomes bold for higher missing rate.
The discriminator of IGANI does not benefit from such a hint while training, instead it is trained to recognize which data is imputed twice without any hint of the masks.
Therefore, the superiority of IGANI compared with MisGAN and GAIN is attributed to its simple architecture and mature discriminator.

\begin{figure}[h]
	\centering
	\includegraphics[width=3.5in]{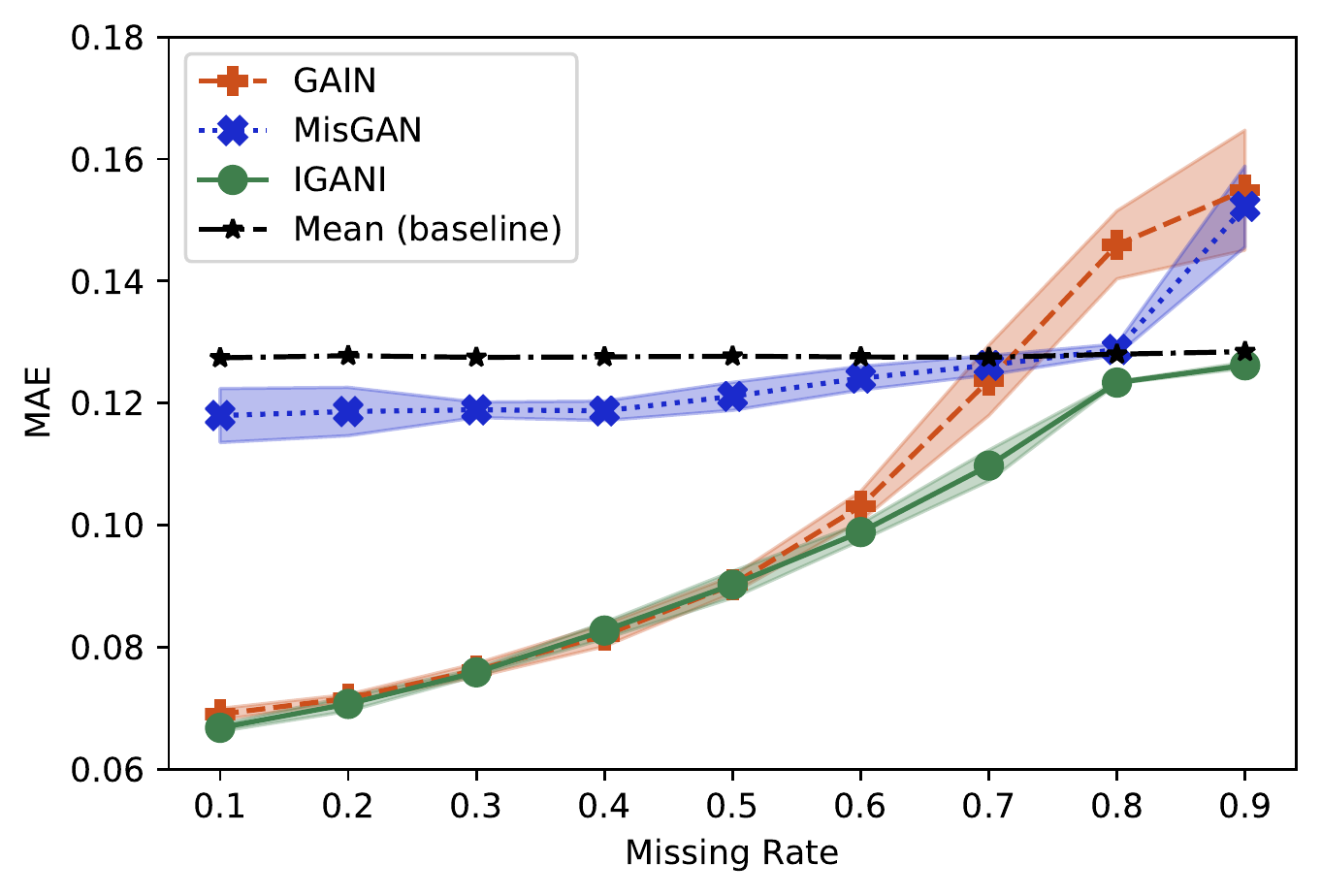}
	\caption{MAE of normalized imputed values by GAIN, MisGAN, and IGANI where shaded area represents $\pm 3\sigma$ for 5 models. The baseline, mean imputation, shown by black star marker, does not outperform IGANI even for higher missing rates.}\label{fig.maeimp}
\end{figure}

\begin{figure}[h]
	\centering
	\includegraphics[width =\textwidth]{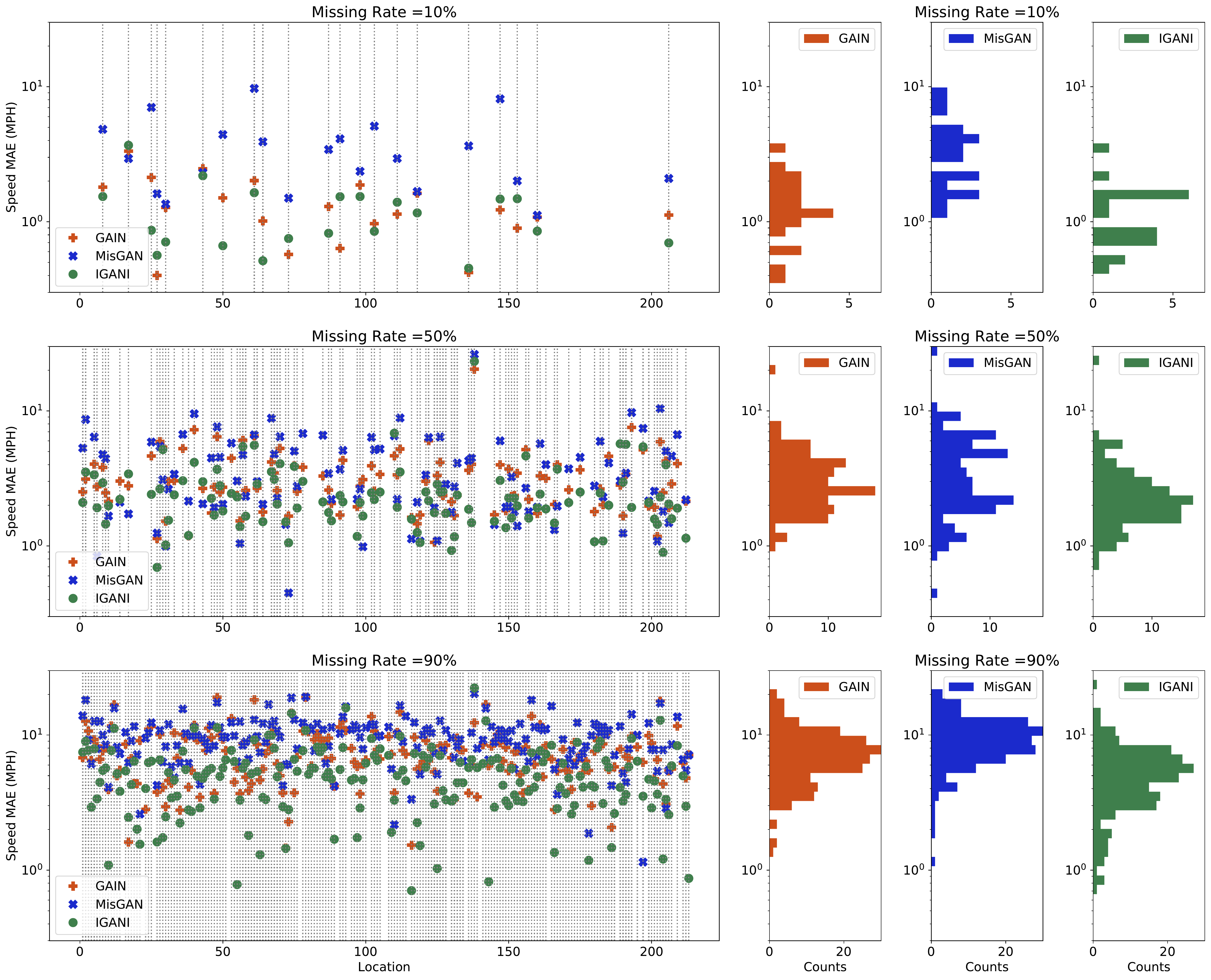}
	\caption{Logarithmic MAE of imputations by GAIN, MisGAN, and IGANI for an individual sample. IGANI outperforms GAIN and MisGAN especially for higher missing rates. Vertical lines on left subplots represent missing value indices. The histograms of MAE for imputed values are also plotted.}\label{fig.maeimpinstance}
\end{figure}

\subsubsection{Short-term traffic prediction using imputed data}
In addition to assessing the accuracy of data imputation, it is also important to evaluate how imputation methods, which are unsupervised learning, perform when used in subsequent analyses or predictions. In this study, we choose short-term traffic prediction using supervised learning-based neural network using imputed data.
In particular, the time interval between two subsequent data samples is  considered to be 10 minutes and the aim is to predict an instance (the speed)  from its preceding one using a fully connected neural network model.
The authors admit that more advanced methods are available for predicting temporal data (like recurrent neural networks (RNN) and its variants, e.g. long short-term memory (LSTM) and gated recurrent unit (GRU)). But, the choice of prediction model is beyond the aim of this work, and our aim is to assess how the errors cause by different imputation methods  transcend into \emph{a} subsequent prediction task. 

\begin{figure*}[h]
	\centering
	\includegraphics[width=\textwidth]{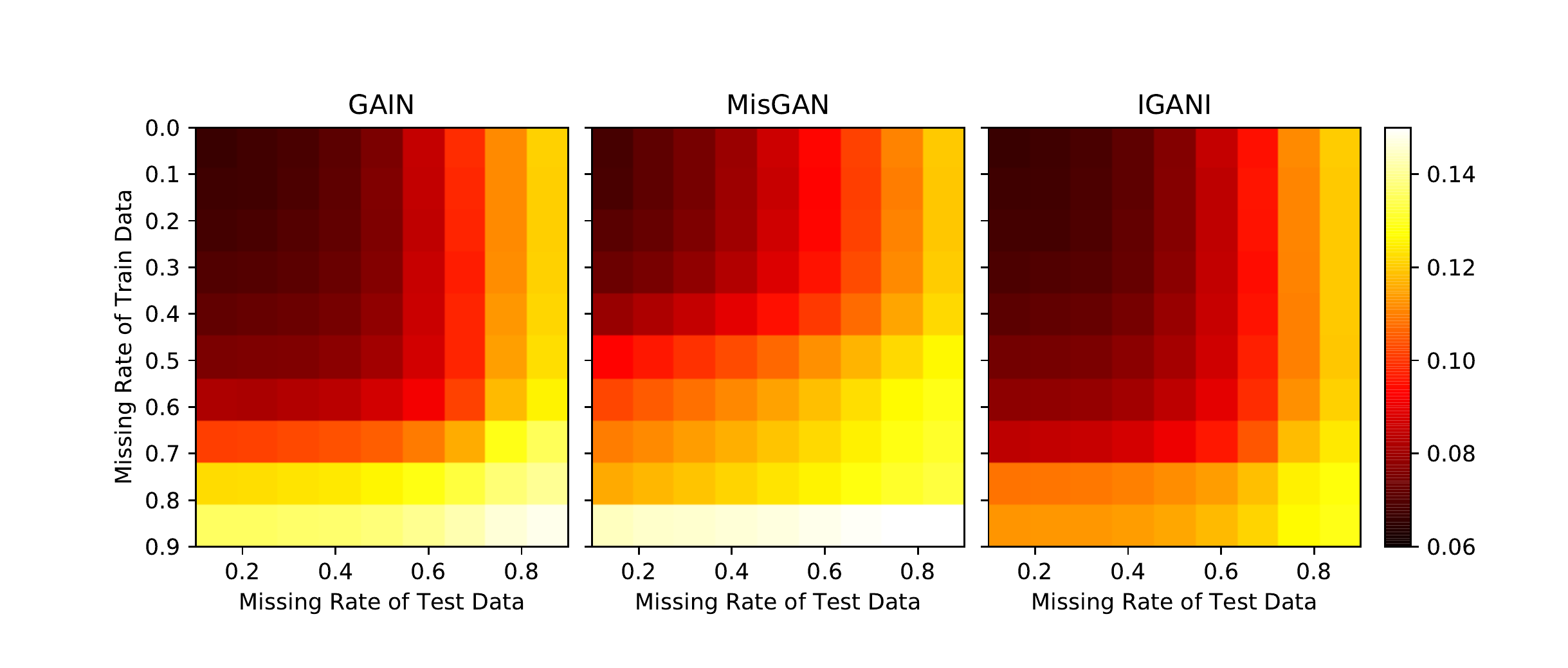}
	\caption{MAE of normalized short-term traffic predictions using missing data imputed by GAIN, MisGAN, and IGANI, averaged over 5 trained imputers for each architecture.}
	\label{fig.maepredmatrix}
\end{figure*}

\begin{figure}[h]
	\centering
	\includegraphics[width=3.5in]{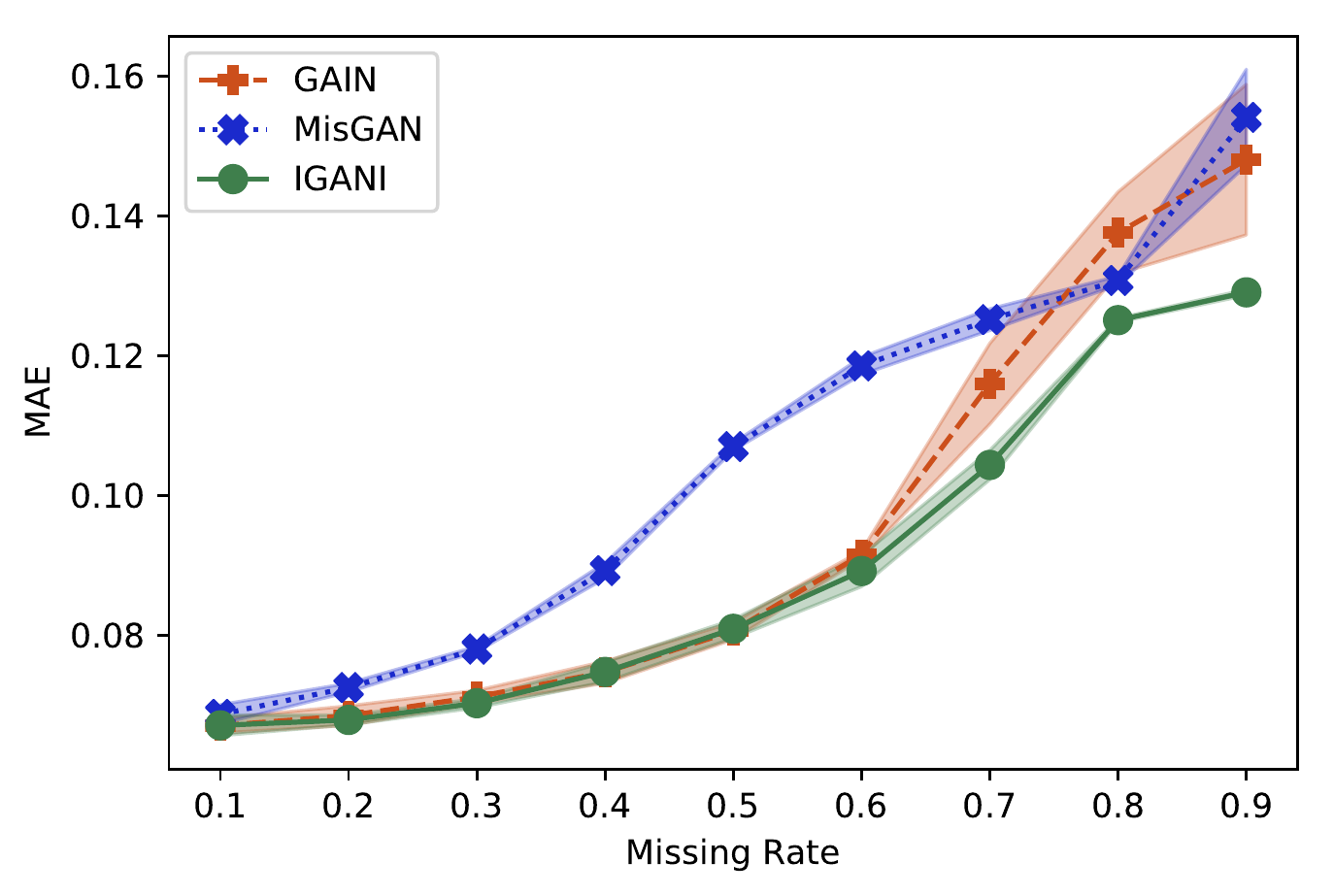}
	\caption{MAE of normalized short-term speed predictions based on imputed data by GAIN, MisGAN, and IGANI. Shaded area represents $\pm 3\sigma$ for the 5  trained imputers. Training and testing data have the same missing rate. }\label{fig.maepred}
\end{figure}

In this work, short-term prediction models are trained and tested for different missing rates. 
These rates range from 0 to 0.9 for training and from 0.1 to 0.9 for testing models.
The imputation methods are GAIN, MisGAN, and IGANI which implies a total of $10\times 9 \times 3=270$ models.
The MAE of testing models for predicting normalized data is observed for all cases in Figs. \ref{fig.maepredmatrix} and \ref{fig.maepred}.
As can be seen from these results, the superiority of IGANI is more significant for higher missing rates where both GAIN and MisGAN lose robustness in supervised learning-based tasks that use imputed data.
Generally, IGANI is more accurate and stable compared with previous GAN-based imputation methods and is strictly recommended for missing rates higher than 50\%.

\subsection{Portland-Vancouver metropolitan region: volume, occupancy, and speed}\label{sec.portland}

In the previous experiment, a single variable of traffic condition, i.e. speed, was imputed.
Nowadays, traffic condition is measured by detectors which are capable of capturing more than one variable.
Dual loop detectors are a common type of traffic detector which measure volume, occupancy and speed data.
While dysfunctions such as communication errors or hardware issues may lead to missing data from dual loop detectors, faults such as chattering, pulse breakup, and hanging on/off may produce low fidelity data \citep{ariannezhad2020large}.
As all variable are not available and valid at the same time and location, imputation of multi-variable traffic data is of critical importance.

In this experiment, a dateset composed of volume, occupancy and speed is used to evaluate our proposed method and other GAN-based imputers. 
We considered different missing rates for volume, occupancy and speed in order to study the mutual  effect of measured variables on the imputation accuracy of the missing one.

\subsubsection{Data description}
The data used in this section is collected from the dual loop detectors of highways in Portland-Vancouver Metropolitan region \citep{portland}.
The time span includes the first half of January 2021 with a time resolution of 15 minutes for detector IDs between 5000 and 6000.
Detectors and time spots which are entirely missing are removed from data, after which the original missing rates for volume, occupancy, and speed were found to be 0, 0, and 0.5\%, respectively.
Similarly to the previous section, the small number of missing speed data are calculated using iterative imputer to form the ``complete" data. This data consists of three matrices of the size $1500 \times 160$ for volume, occupancy, and speed measured at 160 locations (detectors) and 1500 time points.
Out of the total data length (1500 time points), 85\% of the data is used to train the imputer and 15\% is used to test it.

\subsubsection{Architectures}
The imputation is performed with the same DNN layers,  as listed in Table (\ref{table}), to evaluate the performance of architectures under fixed model capacity when only the input dimension increases.
The only difference is that in this case study, the dimension of the input  is 480, instead of 214.

\subsubsection{Imputation of missing traffic data}
In this case study, we consider four missing rates of 0.2, 0.4, 0.6 and 0.8 for each of the three variable (i.e. volume, occupancy, and speed). The imputation results for the $4^3$ combinations of these rates are shown in Figs. \ref{fig.vmr} - \ref{fig.smr} for different GAN-based imputation methods.
As can be seen, IGANI significantly outperforms GAIN and MisGAN whose performance are adversely affected in this case where the dimension of input has increased from 214 to 480.
Also, the expected error increase in higher missing rates can hardly be seen in MisGAN results.
The observed  inaccuracy for GAIN and MisGAN models, as well as the erratic error behavior for MisGAN results can be due to the fact that the hint mechanism of GAIN and complexity of MisGAN architecture have become inefficient now that we are applying them, with unchanged architecture and model capacity, to a new dataset with a different dimension. This shows that IGANI architecture is more robust and can show similar performance levels for a larger data set, while the architecgture of GAIN and Mis GAN  may need to be re-designed.

The results obtained from the IGANI imputation show that the change in the missing rates of the three variables affects the imputation accuracy of  other variables.
This is more obvious in Figs. \ref{fig.vmr} and \ref{fig.omr}, which show that the imputation accuracy of volume and occupancy depends on the missing rates of one another. This effect is less apparent in \ref{fig.smr}, however, where the speed imputation accuracy is not significantly  influenced by the availability of volume and occupancy data.
Such a behavior is probably because speed variation is not significant for this set of highway data, which makes the speed imputation an easier task, regardless of how incomplete the volume and occupancy data are.

\begin{figure*}[h]
	\centering
	\includegraphics[width=\textwidth,height=11cm,keepaspectratio]{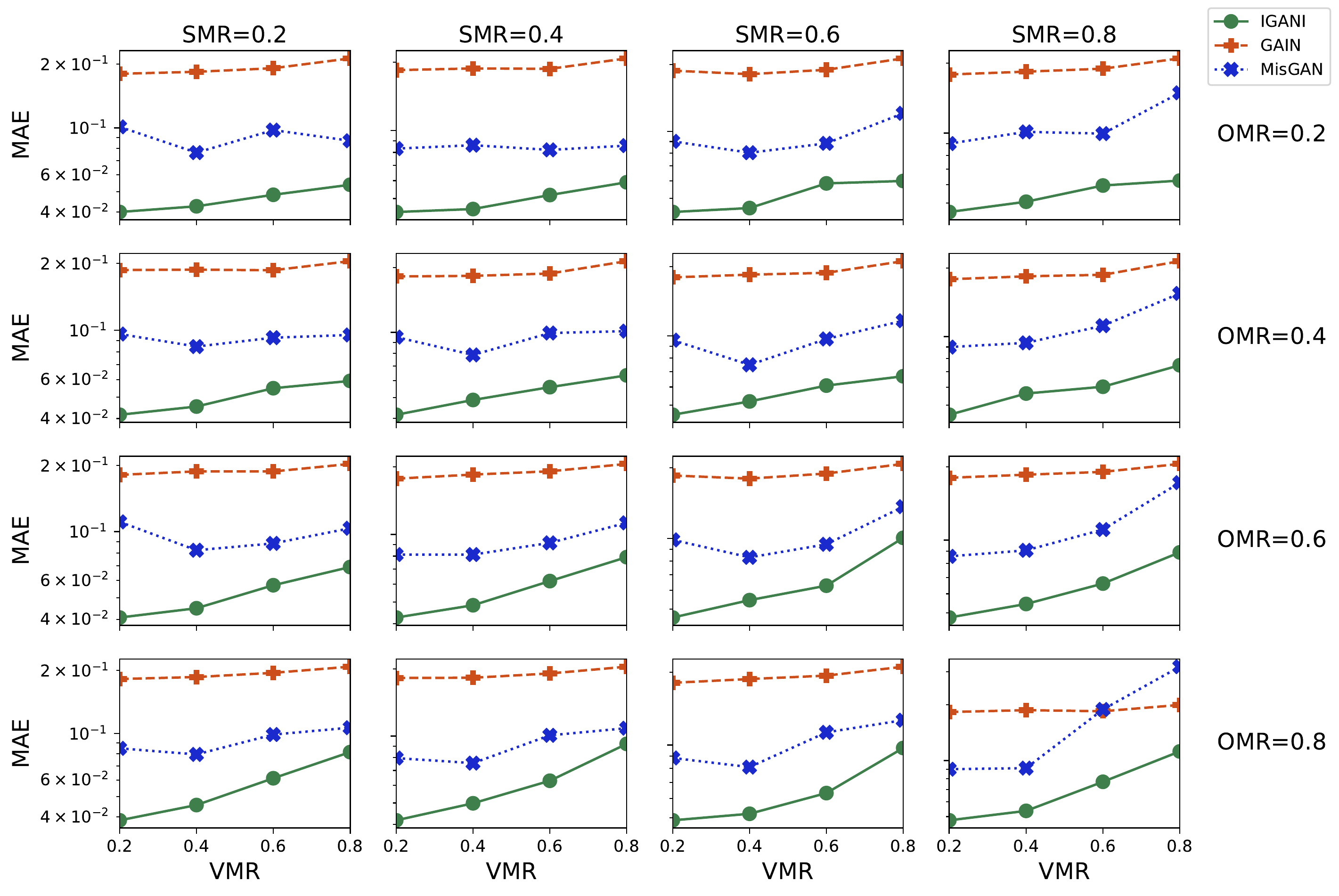}
	\caption{MAE of normalized imputed \textit{volume} data by IGANI, GAIN, and MisGAN. VMR, OMR, and SMR denote volume, occupancy, and speed missing rates, respectively.}
	\label{fig.vmr}
\end{figure*}

\begin{figure*}[h]
	\centering
	\includegraphics[width=\textwidth,height=11cm,keepaspectratio]{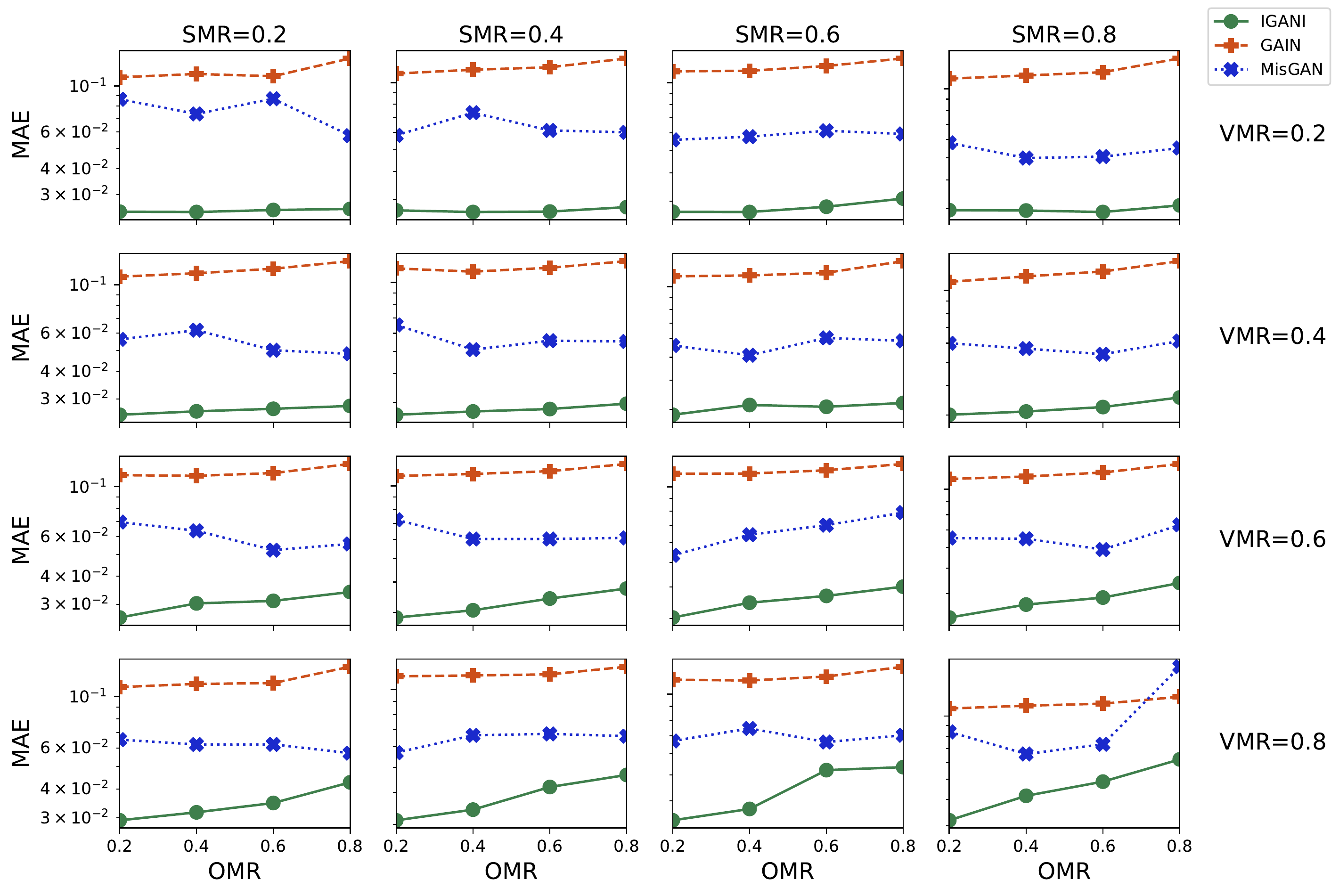}
	\caption{MAE of normalized imputed \textit{occupancy} data by IGANI, GAIN, and MisGAN. VMR, OMR, and SMR denote volume, occupancy, and speed missing rates, respectively.}
	\label{fig.omr}
\end{figure*}

\begin{figure*}[h]
	\centering
	\includegraphics[width=\textwidth,height=11cm,keepaspectratio]{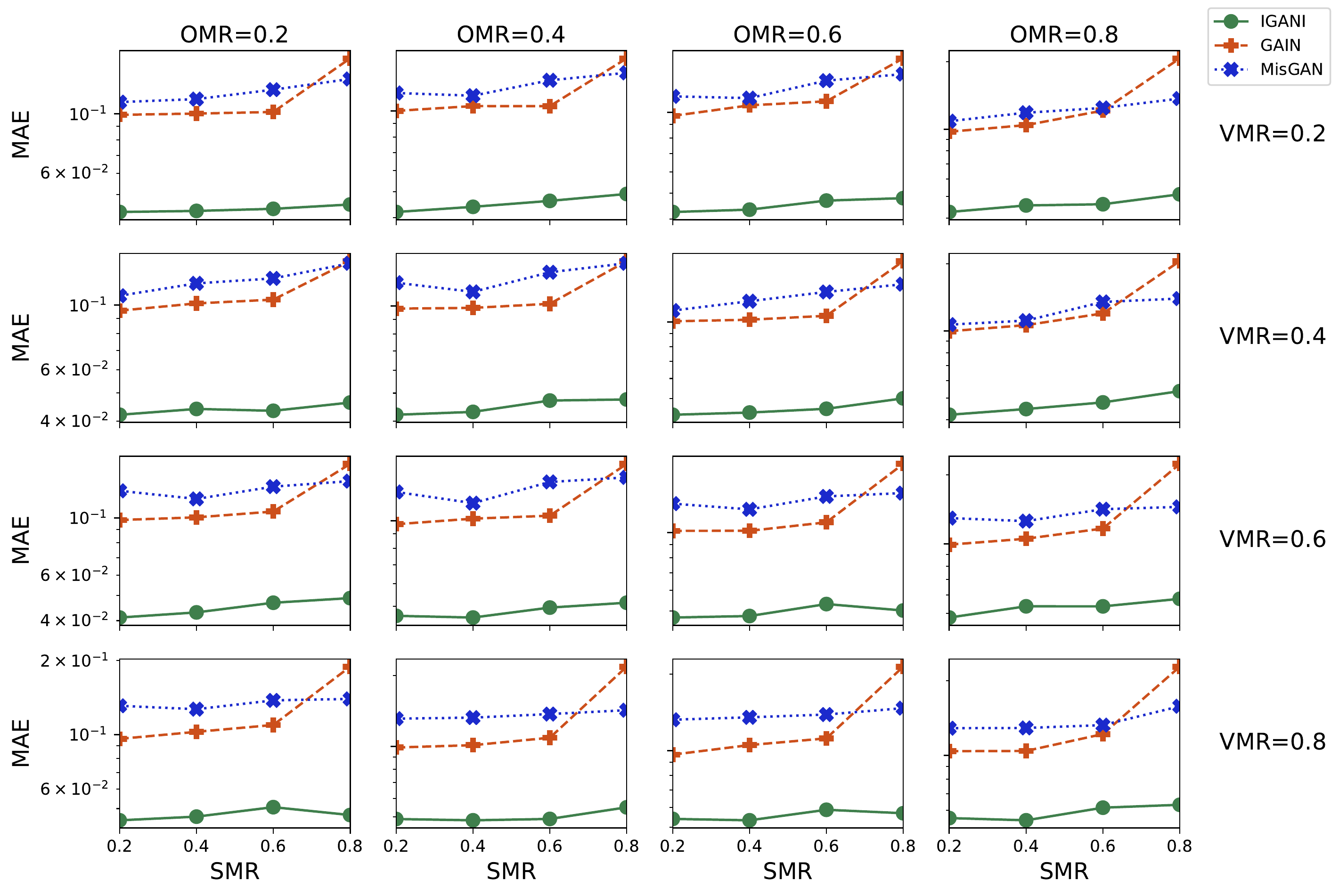}
	\caption{MAE of normalized imputed \textit{speed} data by IGANI, GAIN, and MisGAN. VMR, OMR, and SMR denote volume, occupancy, and speed missing rates, respectively.}
	\label{fig.smr}
\end{figure*}

\section{Conclusion}

In this work, a new GAN architecture, named IGANI, is introduced for data imputation and its performance is evaluated on imputation of missing traffic data and also short-term traffic prediction. It is shown that IGANI significantly outperforms the previous GAN-based imputation architectures (like GAIN and MisGAN) accuracy. It is also shown that when IGANI-imputed data is used in a supervised learning framework to train  short-term traffic predictions, the prediction accuracy is higher compared to the cases where GAIN- or MisGAN-imputed data is used.
The proposed architecture is especially instrumental for the imputation of big data, such as traffic data generated in transportation systems.
This is because  IGANI imputes traffic data with a higher accuracy compared with other GAN-based methods at various missing rates, and as opposed to clustering-based imputation methods (like KNN) does not require searching within a large pool of data to find closest neighbors for imputation.
In case of multi-variable data, like those obtained from traffic loop detectors, IGANI outperforms GAIN and MisGAN, with respect to both accuracy and behavior.
This outperformance is demonstrated through a dataset of volume, occupancy, and speed which studies the effect of missing rate of one variable on the imputation accuracy of the other one.

\section*{Acknowledgment}
This material is based in part upon work supported by the 
National Science Foundation under Grant No.  CMMI-1752302 and USDOT under Grant No. 69A3551747105. 

\bibliographystyle{unsrtnat}
\bibliography{references}  






\end{document}